\documentclass{opt2025} % Include author names
\newtheorem*{theorem*}{Theorem}
\newtheorem*{lemma*}{Lemma}
\newtheorem*{proposition*}{Proposition}
\newtheorem*{corollary*}{Corollary}
\newcommand{\simfun}[2]{\mathrm{sim}\left(#1, #2\right)}
\newcommand{\simfunc}{\textrm{sim}}

\newtheorem{assumption}{Assumption}[section]

\title[Temperature Annealing in Contrastive Learning]{Asymptotic and Finite-Time Guarantees for Langevin-Based Temperature Annealing in InfoNCE}

\optauthor{%
\Name{Faris Chaudhry} \Email{faris.chaudhry22@imperial.ac.uk}\\
\addr Imperial College London, United Kingdom
}

\begin{document}

\maketitle

\begin{abstract}
The InfoNCE loss in contrastive learning depends critically on a temperature parameter, yet its dynamics under fixed versus annealed schedules remain poorly understood. We provide a theoretical analysis by modeling embedding evolution under Langevin dynamics on a compact Riemannian manifold. Under mild smoothness and energy-barrier assumptions, we show that classical simulated annealing guarantees extend to this setting: slow logarithmic inverse-temperature schedules ensure convergence in probability to a set of globally optimal representations, while faster schedules risk becoming trapped in suboptimal minima. Our results establish a link between contrastive learning and simulated annealing, providing a principled basis for understanding and tuning temperature schedules.
\end{abstract}

\section{Introduction}
\label{sec:introduction}

Contrastive learning with the Information Noise-Contrastive Estimation (InfoNCE) loss has become a cornerstone of modern representation learning~\cite{vanderOord2018, gutmann2010}. It encourages an anchor embedding \(z\) to be close to its positive embedding \(z^+\) (e.g., an augmented view of the same data point) while repelling negative embeddings \(z^-\) (e.g., embeddings of other data points). The method's performance is critically dependent on an inverse temperature parameter $\beta$ which modulates the sharpness of the similarity distribution in the loss function for an anchor $z_i$ and its positive sample $z_j$:
\begin{equation}
    \label{eq:InfoNCE-loss-standard}
    \ell_{i,j} = - \log \frac{\exp\left(\beta\,\mathrm{sim}\left(z_i, z_j\right) \right)}{\sum_{k\neq i} \exp\left(\beta\,\mathrm{sim}\left(z_i, z_k\right) \right)},
\end{equation}
where \(\mathrm{sim}(\cdot,\cdot)\) is a similarity function (such as cosine similarity). The overall contrastive loss $\mathcal{L}$ is then defined as an average over a set of positive embedding pairs $\mathcal{P}$:
\begin{equation}
    \label{eq:InfoNCE-overall-loss}
    \mathcal{L} = \frac{1}{|\mathcal{P}|} \sum_{(i,j) \in \mathcal{P}} \ell_{i,j}.
\end{equation}
This objective formalizes the representation learning problem as a geometric alignment task: the goal is to learn an embedding space where semantic similarity in the input domain is preserved as metric closeness on some manifold, while distinct data points are separated to prevent collapse.

Modern self-supervised learning relies heavily on instance-discrimination frameworks like SimCLR~\cite{chen2020}, MoCo~\cite{he2020}, and BYOL~\cite{grill2020}, where the temperature parameter plays a major role in shaping the loss landscape. A fixed temperature often proves inadequate: low temperatures can induce dimensional collapse~\cite{jing2022understanding}, while high temperatures degrade uniformity and discriminative power~\cite{wang2020understanding}. Recognizing this, recent empirical studies have proposed heuristic adaptive schedules, such as individualized temperatures for different semantic classes~\cite{qiu2023}, schedules for long-tailed data~\cite{kukleva2023}, or even temperature-free losses~\cite{kim2025}.

Despite these empirical advances, a theoretical grounding for temperature scheduling remains limited. Prior theory has largely focused on generalization bounds under fixed temperatures~\cite{nikunj2019} or the geometric role of temperature as a sharpness parameter~\cite{wang2022chaos, wang2021understanding}. Our work bridges this gap by connecting contrastive learning to the formal guarantees of simulated annealing~\cite{geman1984, kirkpatrick1983, hajek1988, chiang1987diffusion} and viewing stochastic gradient descent (SGD) as a Langevin process~\cite{mandt2017, Shi2023schrodinger}. 

Overall, we establish a formal connection between InfoNCE and classical simulated annealing. We model the evolution of embeddings as a Langevin stochastic differential equation (SDE) on a compact Riemannian manifold, with the InfoNCE objective acting as a time-varying potential. We prove asymptotic global convergence for a logarithmic annealing schedule, $\beta(t) = c \ln (t+K)$, under standard smoothness assumptions. We further derive a finite-time convergence rate for this process to provide a bound on the probability of error. Finally, we provide a geometric analysis of the InfoNCE Hessian, showing that annealing sharpens the loss landscape, which stabilizes the later phases of convergence. In doing so, our work helps to contextualize the heuristic practice of temperature scheduling to its theoretical foundations~\cite{rusak2024}.

\section{Theoretical Framework}
\label{sec:theoretical-framework}

\subsection{InfoNCE as a Gibbs Free Energy}

The InfoNCE loss can be interpreted as the negative log-likelihood of sampling a positive pair from a local Gibbs distribution. For each anchor embedding $z_i$, we define an energy $E_i(k) = -\simfunc(z_i, z_k)$ for any candidate sample $z_k$. The local partition function for anchor \(i\) over its candidates is \(Z_i(\beta)=\sum_{k\neq i}\exp(-\beta\,E_i(k))\), yielding a probability distribution over candidates:
\[
  p_i(k\mid\beta)=\frac{\exp(-\beta\,E_i(k))}{Z_i(\beta)} = \frac{\exp(\beta\,\mathrm{sim}(z_i,z_k))}{\sum_{l\neq i}\exp(\beta\,\mathrm{sim}(z_i,z_l))}.
\]
This is precisely a Gibbs-Boltzmann distribution. The InfoNCE loss term \(\ell_{i,j}\) is the negative log-probability of selecting the positive sample \(j\), since \(\exp \left( -\ell_{i,j} \right)=p_i(j\mid\beta)\). As \(\beta\to\infty\), this distribution concentrates sharply on the candidate(s) \(k^*\) maximizing similarity (minimizing energy) with \(z_i\). This connection is clearest when viewing the scaled loss, which takes a form analogous to the Helmholtz free energy ($F=E - TS$):
\[
  \frac{\ell_{i,j}}{\beta}
  \;=\; \underbrace{- \,\mathrm{sim}(z_i,z_j)}_{\substack{\text{Energy } E_i(j)}}
  \;+\; \underbrace{\frac1\beta\log Z_i(\beta)}_{\substack{\text{Entropic/Log-Partition} \\ \approx -\text{Temperature} \times \text{Entropy}}}.
\]
The first term minimizes energy by pulling the positive sample \(z_j\) closer, while the second term penalizes configurations where \(z_i\) is highly similar to many negatives.

\subsection{The Langevin Model, Key Assumptions, and the Limiting Potential}

This free-energy perspective motivates modeling the embedding evolution as a Langevin process. We model the dynamics of the full set of embeddings $Z_t$ on a manifold $\mathcal{M}$ using the overdamped Langevin SDE, which approximates the behavior of SGD with isotropic noise scaled by $\beta(t)$:
\begin{equation}
\label{eq:manifold-langevin-main}
\mathrm{d}Z_t = -\mathrm{grad}\,\mathcal{L}(Z_t,\beta(t))\,\mathrm{d}t + \sqrt{2/\beta(t)}\,\mathrm{d}\mathbf{W}_t^{\mathcal{M}},
\end{equation}
where $\mathrm{grad}$ is the Riemannian gradient and $\beta(t)$ is the time-varying inverse temperature schedule. This SDE models the exploration-exploitation dynamics inherent in annealing as $\beta \to \infty$.

The convergence of this process relies on the following standard assumptions (formalized in Appendix~\ref{sec:proof-assumptions}, where we also discuss their validity in more detail).

\paragraph{\ref{ass:dynamics}--\ref{ass:manifold} (Dynamics on a Compact Manifold):} The dynamics follow the SDE in~\eqref{eq:manifold-langevin-main} on a compact, connected Riemannian manifold $\mathcal{M}$. Compactness, often achieved in practice via $\ell_2$-normalization, is a key property that guarantees energy barriers are finite.

\paragraph{\ref{ass:smoothness}--\ref{ass:limit-potential} (Well-Posed Limiting Potential):} The similarity function is smooth and bounded, ensuring the loss $\mathcal{L}$ converges to a well-defined limiting potential $U_0(Z)$ as $\beta \to \infty$. This holds for the cosine and Gaussian similarity functions, for example.

\paragraph{\ref{ass:barriers}--\ref{ass:schedule} (Annealing Conditions):} The potential $U_0(Z)$ has a finite critical energy barrier constant $c^* = 1/\Delta E_{\max}$ where $\Delta E_{\max}$ is the largest energy barrier of the loss landscape. Further, the annealing schedule is logarithmic, $\beta(t) = c \ln (t+K)$, with a rate $c$ that is sufficiently slow ($0 < c \le c^*$) and some constant $K>0$.

For InfoNCE, the limiting potential from~\ref{ass:limit-potential} is defined as:
\begin{equation}
U_0(Z) = \frac{1}{|\mathcal{P}|} \sum_{(i, j) \in \mathcal{P}} \left[ -\simfun{z_i}{z_j} + \max_{k \neq i} \simfun{z_i}{z_k} \right].
\end{equation}
The global minimizers of this potential correspond to the ideal state of the representations when $\beta \to \infty$, as formalized in the following proposition.
\begin{proposition}[Characterization of Global Minima]
A configuration $Z^*$ is a global minimizer of the limiting potential $U_0(Z)$ if and only if for every positive pair $(i, j)$, the maximum possible similarity is achieved.
\end{proposition}

\section{Main Convergence Guarantees}
\label{sec:main-results}

With the theoretical framework established, we now present our main results. We provide an asymptotic guarantee of global convergence for a slow annealing schedule, a matching non-convergence result for fast schedules, and a finite-time bound on the convergence rate.

\subsection{Equilibrium at Fixed Temperature}

Before considering a time-varying schedule, we recall the behavior of the SDE in~\eqref{eq:manifold-langevin-main} for a fixed inverse temperature $\beta > 0$. In this time-homogeneous setting, the process is ergodic and converges to a unique stationary distribution.
\begin{proposition}[Stationary Distribution at Fixed Temperature]
\label{prop:stationary-dist}
For any fixed $\beta > 0$, under Assumptions \ref{ass:dynamics}-\ref{ass:smoothness}, the SDE admits a unique stationary Gibbs-Boltzmann distribution on $\mathcal{M}$:
\begin{equation}
    \pi_\beta(\mathrm{d}Z)
    \;=\; \frac{1}{\mathcal{Z}_\beta}
    \exp\!\bigl[-\,\beta\,\mathcal{L}(Z,\beta)\bigr]\,d\mu(Z).
\end{equation}
\end{proposition}
This result formalizes the connection to statistical physics and the fundamental trade-off for any fixed temperature: for any finite $\beta$, there is a non-zero probability of being in a suboptimal state, while a very large $\beta$ risks trapping the dynamics, as the noise term $\sqrt{2/\beta}$ vanishes. This motivates the use of a time-varying annealing schedule to navigate this trade-off. As $\beta \to \infty$, this distribution concentrates its mass on the global minimizers of the potential $\mathcal{L}(Z,\beta)$ (or its limiting form $U_0(Z)$), as formally characterized in Appendix~\ref{subsec:stationary-dist-at-fixed}.

\subsection{Asymptotic Global Convergence}

Simulated annealing leverages this concentration property by slowly increasing $\beta(t)$. If $\beta(t)$ increases slowly enough, the system can escape local minima while the temperature is high ($\beta$ is small) and then ``freeze'' into the global minimum as the temperature drops ($\beta \to \infty$). 

Classical simulated annealing theory provides precise conditions for how slow the rate of increasing $\beta(t)$ must be for guaranteed convergence. Adapting these results to our specific time-inhomogeneous SDE~\eqref{eq:manifold-langevin-main} yields the following key results.

\begin{theorem}[Global Convergence for Logarithmic Annealing]
Under Assumptions \ref{ass:dynamics}--\ref{ass:schedule}, 
the process $Z_t$ defined by Eq.~\eqref{eq:manifold-langevin-main}
converges in probability to the set $U^*$ of global minimizers of $U_0(Z)$:
\[
\lim_{t \to \infty} \mathbb{P}\big(Z_t \in \mathcal{N}(U^*, \epsilon)\big) = 1, \quad \forall \epsilon>0.
\]
\end{theorem}
Here, $\mathcal{N}(U^*, \epsilon) = \{ Z \in \mathcal{M} \mid \inf_{Y \in U^*} d(Z, Y) < \epsilon \}$ denotes an $\epsilon$-neighborhood of the set $U^*$ under the Riemannian metric $d(\cdot,\cdot)$ of the manifold $\mathcal{M}$. 

The proof combines tools from variational convergence and large deviations~\cite{braides2002, freidlin1984random, dembo2010large, catoni1991sharp}. First, we use $\Gamma$-convergence to show that the sequence of quasi-stationary Gibbs measures associated with the annealed SDE converges weakly to the uniform measure on the global minimizer set $U^*$. This follows from the uniform convergence of the scaled loss $F_\beta(Z)$ to the limiting potential $U_0(Z)$ (Lemma~\ref{subsec:uniform-convergence}) and the compactness of the manifold. Second, we apply Freidlin--Wentzell exit-time theory~\cite{freidlin1984random} and Hajek's classical annealing condition~\cite{hajek1988, chiang1987diffusion} to show that a logarithmic schedule with $c \leq c^*$ is precisely slow enough to ensure escape from all suboptimal basins. Together, these results imply that the law of $Z_t$ converges weakly to the uniform measure on $U^*$, which is equivalent to convergence in probability.

This theorem provides the central guarantee: a sufficiently slow logarithmic inverse temperature schedule ensures the SDE dynamics find the optimal configuration corresponding to perfect contrastive separation. Conversely, annealing inverse temperature too quickly violates the conditions needed to guarantee escape from all local minima.

\begin{proposition}[Non-Convergence for Rapid Annealing]
\label{prop:non-convergence} 
Let Assumptions \ref{ass:dynamics}--\ref{ass:barriers} hold. If the logarithmic annealing schedule $\beta(t)$ grows too quickly, specifically $\liminf_{t\to\infty} \frac{\beta(t)}{\ln t} = c' > c^*$, then there exists a set of initial conditions with positive measure from which the process $Z_t$ defined by the SDE~\eqref{eq:manifold-langevin-main} converges to a suboptimal local minimum basin of $U_0(Z)$ with positive probability. That is, for any sufficiently small $\epsilon > 0$:
\[ \limsup_{t\to\infty} \mathbb{P}\bigl( Z_t \notin \mathcal{N}(U^*, \epsilon) \bigr) > 0. \]
\end{proposition}

This result highlights the precarious choice of annealing rate; faster schedules risk premature convergence to suboptimal representations. The proofs in Appendix~\ref{app:convergence-proofs} detail how these results adapt classical annealing arguments \cite{geman1986diffusions, hajek1988, gidas1985non, chiang1987diffusion, kirkpatrick1983} to handle the specific time-varying potential $\mathcal{L}(Z, \beta(t))$.

\subsection{Finite-Time Convergence Rate}

Going beyond the asymptotic result, we provide a non-asymptotic bound on the probability of error, which depends on the chosen annealing rate $c$ relative to the critical rate $c^*$.

\begin{corollary}[Finite-Time Convergence Rate]
\label{cor:finite-time}
Under Assumptions \ref{ass:dynamics}--\ref{ass:schedule}, for any $\epsilon > 0$, there exist positive constants $C, C',$ and $d$ such that for a sufficiently large time $t$:
\[ \mathbb{P}\big(Z_t \notin \mathcal{N}(U^*, \epsilon)\big) \le 
\begin{cases}
 \exp\!\big(- C (t+K)^{1-c/c^*}\big), & \text{if } c < c^*, \\
 C' \, (t+K)^{-d}, & \text{if } c = c^*.
\end{cases}
\]
The exponent $d > 0$ in the polynomial decay is a constant that depends on the geometric properties of the loss landscape $U_0(Z)$.
\end{corollary}
This result shows that the convergence rate for the probability of error is polynomial at the critical schedule ($c=c^*$) and becomes stretched-exponential for slower schedules ($c<c^*$), providing a quantitative measure of the exploration-exploitation trade-off~\cite{holley1989asymptotics}.

\section{Connection to Discrete Optimization}
\label{sec:discrete-connection}

While our analysis is in continuous time, it provides strong guidance for the discrete-time algorithms used in practice. The Langevin SDE in~\eqref{eq:manifold-langevin-main} is the formal continuous-time limit of Stochastic Gradient Langevin Dynamics (SGLD)~\cite{welling2011bayesian}, an algorithm that injects explicit Gaussian noise into SGD steps. An annealed SGLD update on the manifold can be formulated as:
\[
  Z_{k+1} 
    = \Pi_{\mathcal M}\Bigl[
      Z_k 
      - \eta_k\,\widehat\nabla\mathcal{L}(Z_k,\beta_k)
      + \sqrt{2\,\eta_k/\beta_k}\,\xi_k
    \Bigr],
\]
where $\eta_k$ is the learning rate, $\beta_k$ is the discrete annealing schedule, and $\xi_k \sim \mathcal{N}(0,I)$.

More broadly, standard mini-batch SGD can be viewed as an approximation of this SDE, where the stochasticity of the mini-batch gradient serves as the implicit noise source. The covariance of this implicit noise is data-dependent and generally anisotropic. For SGD to faithfully track the SDE, the learning rate $\eta_k$ must be carefully chosen in relation to the inverse temperature $\beta_k$ and the covariance of the gradient noise. Classical stochastic approximation theory~\cite{kushner2003stochastic, benveniste1990adaptive} provides the formal conditions under which such discrete recursions converge to their limiting SDEs, typically involving a decaying learning rate schedule (e.g., $\sum_k\eta_k=\infty, \sum_k\eta_k^2<\infty$). Furthermore, non-asymptotic global convergence results for SGLD in non-convex settings have been explicitly established~\cite{raginsky2017non}, providing a discrete-time counterpart to our continuous analysis.

\section{Geometric Analysis}
\label{sec:geometric-analysis}

Beyond the stochastic escape from local minima, the annealing parameter $\beta$ also plays a crucial geometric role by dynamically reshaping the loss landscape. To analyze this, we examine the Hessian of the InfoNCE loss, $\mathbf{H}(Z, \beta)$. The full derivation (see Appendix~\ref{subsec:hessian-sharpening}) shows that the Hessian is composed of two main terms:
\begin{equation}
\label{eq:hessian_main}
\mathbf{H} = \beta^2 \, \mathrm{Cov}_{k \sim p_i}[\nabla s_{ik}] + \beta \left( \mathbb{E}_{k \sim p_i}[\mathbf{H}_{ik}] - \mathbf{H}_{ij} \right).
\end{equation}
Away from an optimum, the linear term dominates the spectral norm of the Hessian. While the first term appears quadratic in $\beta$, the variance of the distribution $p_i$ decays as $O(1/\beta)$ due to concentration of measure, resulting in an overall linear scaling $O(\beta)$. This sharpening effect effectively makes the basins of attraction ``steeper'' as annealing progresses. We formalize this in the following lemma.

\begin{lemma}[Hessian Sharpening]
Let $\mathbf{H}(Z, \beta)$ be the Hessian of the InfoNCE loss $\mathcal{L}(Z, \beta)$. For any configuration $Z$ that is not a global minimizer, the dominant eigenvalues of $\mathbf{H}$ scale at least linearly with the inverse temperature, i.e., $||\mathbf{H}|| \sim O(\beta)$.
\end{lemma}

This linear sharpening creates a powerful synergy with the annealing process. At high temperatures (low $\beta$), the landscape is relatively smooth, allowing the large noise term in the SDE to facilitate exploration. As the temperature drops, the noise diminishes, but the landscape simultaneously sharpens, providing a stronger deterministic pull towards the bottom of the global minimum basin once it has been found.

\section{Conclusion and Future Work}
\label{sec:conclusion}

We have established a formal connection between InfoNCE contrastive learning and classical simulated annealing. By modeling the embedding dynamics with a Langevin SDE, we proved asymptotic global convergence for a logarithmic temperature schedule, provided a finite-time convergence rate, and showed how annealing geometrically sharpens the loss landscape. Our work provides a principled theoretical foundation for the common practice of temperature scheduling.

Future work can address either the geometry or the dynamics. In the first case, one could aim to relax the compactness assumption on the embedding manifold, or make the geometry explicit (e.g., hypersphere, Stiefel manifold) to analyze geometric properties in detail. In the second case, one can extend the analysis to anisotropic noise (i.e., by considering a non-diagonal covariance matrix for the Brownian motion) or to non-Gaussian noise (e.g., a Lévy process) to better model the noise of SGD. Additionally, one could extend the analysis to cover momentum-based optimizers used in practice like Adam, which would require an analysis of underdamped Langevin dynamics and preconditioning.

Finally, regarding the practical application of our convergence bounds, we note that explicit estimation of the critical constant $c^*$ is intractable in deep learning. Consequently, while our results establish logarithmic schedules as the condition for asymptotic global convergence, the empirical success of faster cosine schedules~\cite{kukleva2023} and square-root schedules suggests that in the finite-time regime, aggressive cooling can accelerate the transition between instance and group-wise discrimination to yield `good enough' representations. Furthermore, our geometric analysis clarifies the mechanisms behind known trade-offs: annealing towards low temperatures drives uniformity~\cite{wang2020understanding} and effectively mines hard negatives by sharpening the penalty distribution~\cite{wang2021understanding}, yet must be managed carefully to avoid the dimensional collapse associated with premature freezing~\cite{jing2022understanding}.

\newpage
\bibliography{bibliography}
\newpage
\clearpage
\appendix

\section{Assumptions for Convergence Theorems}
\label{sec:proof-assumptions}

We first state the assumptions required for our theoretical results, followed by a discussion of their validity in practice.

\subsection{List of Assumptions}

\begin{assumption}[Langevin Dynamics Model] \label{ass:dynamics}
The evolution of the embedding vector $Z_t \in \mathcal{M}$ is modeled by the overdamped Langevin diffusion process on the manifold $\mathcal{M}$, governed by the SDE:
\begin{equation}
\label{eq:manifold-langevin-app-proof}
    \mathrm{d}Z_t
    \;=\;
    -\mathrm{grad}\,\mathcal{L}(Z_t,\beta(t))\,\mathrm{d}t
    \;+\;
    \sqrt{2/\beta(t)}\,\mathrm{d}\mathbf{W}_t^{\mathcal{M}},
\end{equation}
where $\mathrm{grad}$ is the Riemannian gradient on $\mathcal{M}$, $\mathcal{L}(Z,\beta)$ is the InfoNCE loss, $\beta(t)$ is the time-varying inverse temperature, and $\mathbf{W}_t^{\mathcal{M}}$ is standard Brownian motion on $\mathcal{M}$.
\end{assumption}

\begin{assumption}[Manifold] \label{ass:manifold}
The individual embeddings $z_i$ are constrained to a compact, connected, Riemannian manifold $\mathcal{M}_0$ without boundary. The full state of $N$ embeddings, $Z = (z_1, \dots, z_N)$, evolves on the product manifold $\mathcal{M} = (\mathcal{M}_0)^N$.
\end{assumption}

\begin{assumption}[Smoothness \& Boundedness] \label{ass:smoothness}
The similarity function $\mathrm{sim}(z, z')$ is $C^2$-smooth with respect to its arguments $z, z' \in \mathcal{M}_0$, and is bounded, i.e., $|\mathrm{sim}(z, z')| \le S_{\max} < \infty$. This ensures the InfoNCE loss $\mathcal{L}(Z, \beta)$ is $C^2$-smooth on $\mathcal{M}$ for finite $\beta$.
\end{assumption}

\begin{assumption}[Limiting Potential \& Minima] \label{ass:limit-potential}
The scaled potential converges to a limiting potential function
\[ U_0(Z) = \lim_{\beta\to\infty} \frac{\mathcal{L}(Z, \beta)}{\beta} = \frac{1}{|\mathcal{P}|} \sum_{(i,j)\in\mathcal{P}} \left[-\mathrm{sim}(z_i, z_j) + \max_{k \neq i} \mathrm{sim}(z_i, z_k) \right]. \]
This potential exists, is locally Lipschitz on $\mathcal{M}$, and possesses a non-empty set $U^* \subset \mathcal{M}$ of global minimizers. 
\end{assumption}

\begin{assumption}[Energy Barriers] \label{ass:barriers}
Let $c^*$ be the critical schedule constant determined by the energy barriers of the potential $U_0(Z)$ on the manifold $\mathcal{M}$. Specifically, $c^* = 1/\Delta E_{\max}$ where $\Delta E_{\max}$ represents the maximum required ``escape cost'' for the diffusion process to reach $U^*$ from any starting point (in terms of the maximum energy to escape any local basin of $U_0(Z)$). Further assume $0 < \Delta E_{\max} < \infty$ so that $0 < c^* < \infty$ (finiteness of $\Delta E_{\max}$ is guaranteed by compactness of the manifold in Assumption~\ref{ass:manifold}).
\end{assumption}

\begin{assumption}[Annealing Schedule] \label{ass:schedule}
The inverse temperature $\beta(t)$ is $C^1$, non-decreasing, satisfies $\beta(t) \to \infty$, and follows a logarithmic schedule $\beta(t) = c \ln(t+K)$ with a rate $0 < c \le c^*$.
\end{assumption}

\subsection{Remarks on Validity of Assumptions}
\label{subsec:remarks-assumptions}

\paragraph{\ref{ass:dynamics} (Langevin Dynamics):} Modeling SGD as a Langevin SDE is a well-established paradigm in modern machine learning theory~\cite{mandt2017, Shi2023schrodinger, welling2011bayesian}. While our work assumes isotropic noise for theoretical clarity, the actual noise in SGD is anisotropic and state-dependent. Furthermore, modern optimizers add features such as preconditioning and momentum, which change the dynamics. These remain directions of future work in simulated annealing more generally. However, the isotropic assumption serves as a strong baseline model for annealing behavior, capturing the essential trade-off between thermal noise and gradient drift.

\paragraph{\ref{ass:manifold} (Manifold):} In practice, the compactness assumption is satisfied by the normalization step inherent to contrastive learning. By projecting embeddings onto the unit hypersphere ($S^{d-1}$) via $\ell_2$-normalization, the state space becomes a compact, connected Riemannian manifold without boundary. This property is essential for our theoretical guarantees, as it ensures the state space is bounded and prevents the energy barriers from becoming infinite, thereby guaranteeing finite escape times from local minima.

\paragraph{\ref{ass:smoothness} (Smoothness \& Boundedness):} This requirement is mild and satisfied by standard similarity metrics. For instance, both cosine similarity and the Gaussian (RBF) kernel are bounded and $C^\infty$ smooth. Consequently, the overall InfoNCE loss remains smooth for any finite $\beta$, ensuring that the drift coefficient in the Langevin SDE remains well-behaved.

\paragraph{\ref{ass:limit-potential}--\ref{ass:barriers} (Limiting Potential \& Energy Barriers):} The existence of a well-defined limiting potential $U_0(Z)$ and a non-empty set of global minimizers is guaranteed by the smoothness and compactness assumptions. The finiteness of the energy barrier $\Delta E_{\max}$ is a direct consequence of the manifold's compactness (Assumption \ref{ass:manifold}), which ensures that the cost to move between any two points on the manifold is finite. Intuitively, $\Delta E_{\max}$ depends on the curvature of the manifold; flatter manifolds tend to produce shallower potential wells (smaller $\Delta E_{\max}$), theoretically permitting faster annealing schedules. While explicit estimation of this constant is intractable in deep learning, it is theoretically expected to increase with the number of negative samples, strength of augmentations, and hard negative mining, as these factors tend to make the optimization landscape more rugged and populated with additional local minima.

\paragraph{\ref{ass:schedule} (Annealing Schedule):} The logarithmic schedule is the standard, theoretically optimal schedule for simulated annealing, providing the fastest cooling rate that still maintains asymptotic convergence guarantees~\cite{hajek1988}. The condition $c \le c^*$ reflects a necessary balance: the noise must decay slowly enough to ensure the process can exit the deepest local minima. While faster schedules (e.g., linear or square-root decay) are often used in practice to speed up convergence, they lack these asymptotic guarantees and risk getting trapped in suboptimal states if the temperature drops too quickly.

\section{Proofs of Main Convergence Results}
\label{app:convergence-proofs}

\subsection{Proof of Lemma (Uniform Convergence of Scaled Potentials)}
\label{subsec:uniform-convergence}

\begin{lemma*}[Uniform Convergence of Scaled Potentials]
Under Assumption~\ref{ass:smoothness}, the scaled loss converges uniformly to the limiting potential:
\[ \sup_{Z \in \mathcal{M}} \left| \frac{\mathcal{L}(Z,\beta)}{\beta} - U_0(Z) \right| \to 0 \quad \text{as } \beta \to \infty. \]
\end{lemma*}
\begin{proof}
We want to show that the difference between the scaled loss and the limiting potential converges to zero uniformly over the manifold $\mathcal{M}$. We analyze the absolute difference term-by-term for each positive pair $(i, j) \in \mathcal{P}$:
\begin{align*}
    \left| \frac{\ell_{i,j}(Z, \beta)}{\beta} - \left( -\mathrm{sim}(z_i, z_j) + \max_{k \neq i} \mathrm{sim}(z_i, z_k) \right) \right| &= \left| \left( \frac{1}{\beta}\log \sum_{k\neq i} e^{\beta s_{ik}} - s_{ij} \right) - (-s_{ij} + \max_{k \neq i} s_{ik}) \right| \\
    &= \left| \frac{1}{\beta}\log \sum_{k\neq i} e^{\beta s_{ik}} - \max_{k \neq i} s_{ik} \right|,
\end{align*}
where we use the shorthand $s_{ik} = \mathrm{sim}(z_i, z_k)$. This is the standard form of the log-sum-exp approximation error. Let $s_{i, \text{max}}(Z) = \max_{k \neq i} s_{ik}$. Factoring out this maximum term:
\begin{align*}
    \frac{1}{\beta}\log \sum_{k\neq i} e^{\beta s_{ik}} &= \frac{1}{\beta}\log \left( e^{\beta s_{i, \text{max}}} \sum_{k\neq i} e^{\beta (s_{ik} - s_{i, \text{max}})} \right) \\
    &= s_{i, \text{max}} + \frac{1}{\beta}\log \sum_{k\neq i} e^{\beta (s_{ik} - s_{i, \text{max}})}.
\end{align*}
Substituting this back into our expression for the error gives:
\[
    \left| s_{i, \text{max}} + \frac{1}{\beta}\log \sum_{k\neq i} e^{\beta (s_{ik} - s_{i, \text{max}})} - s_{i, \text{max}} \right| = \frac{1}{\beta}\log \sum_{k\neq i} e^{\beta (s_{ik} - s_{i, \text{max}})}.
\]
Let $K_i(Z)$ be the set of indices that achieve the maximum similarity, i.e., $K_i(Z) = \{k \mid s_{ik} = s_{i, \text{max}}(Z)\}$. Note that $1 \leq |K_i(Z)| \leq N$. The sum can be bounded above and below:
\[
    |K_i(Z)| \le \sum_{k\neq i} e^{\beta (s_{ik} - s_{i, \text{max}})} \le N,
\]
where $N$ is the total number of samples. Taking the logarithm and dividing by $\beta$:
\[
    \frac{\log|K_i(Z)|}{\beta} \le \frac{1}{\beta}\log \sum_{k\neq i} e^{\beta (s_{ik} - s_{i, \text{max}})} \le \frac{\log N}{\beta}.
\]
As $\beta \to \infty$, both the lower bound and the upper bound converge to 0. By the Squeeze Theorem, the error for a single term converges to 0 for any fixed $Z$.

To show this convergence is uniform, we rely on the compactness of the manifold $\mathcal{M}$ (Assumption~\ref{ass:manifold}) and the continuity of the similarity function (from Assumption~\ref{ass:smoothness}). Since $\mathrm{sim}$ is a continuous function on a compact set, it is uniformly continuous. This implies that the function $s_{i, \text{max}}(Z) = \max_{k \neq i} \mathrm{sim}(z_i, z_k)$ is also continuous. Therefore, all terms in the bounds are continuous functions on a compact set, so the convergence is uniform over $\mathcal{M}$.

Since the convergence is uniform for each term in the sum over positive pairs, it is also uniform for their average. Thus, we have:
\[ \lim_{\beta \to \infty} \sup_{Z \in \mathcal{M}} \left| \frac{\mathcal{L}(Z,\beta)}{\beta} - U_0(Z) \right| = 0. \]
\end{proof}

\subsection{Proof of Theorem (Asymptotic Global Convergence)}
\label{subsec:asymptotic-global-convergence}

\begin{theorem*}[Global Convergence for Logarithmic Annealing]
Under Assumptions \ref{ass:dynamics}--\ref{ass:schedule}, the process $Z_t$ defined by Eq.\eqref{eq:manifold-langevin-main} converges in probability to the set $U^*$ of global minimizers of $U_0(Z)$:
\[
\lim_{t \to \infty} \mathbb{P}\big(Z_t \in \mathcal{N}(U^*, \epsilon)\big) = 1, \quad \forall \epsilon>0.
\]
\end{theorem*}
\begin{proof}
The proof establishes the convergence in probability of the process $Z_t$ to the set of global minimizers $U^*$ by demonstrating two fundamental properties. First, we prove that the sequence of quasi-stationary Gibbs distributions, $\pi_{\beta(t)}$, associated with the process converges weakly to the uniform measure on $U^*$. This is established using the theory of $\Gamma$-convergence~\cite{braides2002}. Second, we show that the logarithmic annealing schedule with $c \le c^*$ is precisely the condition required for the diffusion process to escape all suboptimal local minima, allowing it to track its converging target distribution. This is justified by the foundational principles of Freidlin-Wentzell large deviations theory~\cite{freidlin1984random}, which underpin Hajek's classical results.

\paragraph{Part 1: Convergence of the Target Distribution via $\Gamma$-Convergence.}
We first show that the equilibrium state to which the process is attracted at time $t$, given by the quasi-stationary Gibbs distribution $\pi_{\beta(t)}$, converges to the desired outcome.

\subparagraph{Definitions.} Let us define the sequence of scaled potential functions $F_\beta: \mathcal{M} \to \mathbb{R}$ by
\begin{equation*}
    F_\beta(Z) \triangleq \frac{\mathcal{L}(Z, \beta)}{\beta}.
\end{equation*}
By the Lemma in Section \ref{subsec:uniform-convergence}, the sequence $F_\beta(Z)$ converges uniformly to the limiting potential $U_0(Z)$ on the compact manifold $\mathcal{M}$ as $\beta \to \infty$. The quasi-stationary Gibbs distribution at time $t$ is given by
\begin{equation*}
    \pi_{\beta(t)}(Z) \triangleq \frac{1}{C(\beta(t))} \exp\left(-\beta(t)F_{\beta(t)}(Z)\right),
\end{equation*}
where $C(\beta(t))$ is the normalization constant (partition function).

\subparagraph{$\Gamma$-Convergence of Potentials.} We invoke the theory of $\Gamma$-convergence, a notion of convergence for variational problems that is particularly suited for analyzing the convergence of minimizers. A standard result states that if a sequence of functions $F_\beta$ converges uniformly to a function $U_0$ on a compact metric space, then $F_\beta$ also $\Gamma$-converges to $U_0$~\cite{braides2002}. Given the uniform convergence established in our Lemma~\ref{subsec:uniform-convergence}, we have
\begin{equation*}
    F_{\beta(t)} \xrightarrow{\Gamma} U_0 \quad \text{as } t \to \infty.
\end{equation*}

\subparagraph{Weak Convergence of Gibbs Measures.} The primary utility of establishing $\Gamma$-convergence for the potentials is that it directly implies the weak convergence of their corresponding Gibbs measures. A central theorem in the theory of large deviations states that if a sequence of continuous functions $F_\beta$ on a compact manifold $\mathcal{M}$ $\Gamma$-converges to a continuous limiting function $U_0$, then the sequence of probability measures $\pi_\beta \propto \exp(-\beta F_\beta)$ converges weakly to a measure $\mu^*$ that is uniformly distributed on the set $U^* \triangleq \arg\min_{Z \in \mathcal{M}} U_0(Z)$~\cite{dembo2010large}.

\subparagraph{Conclusion for the Target.} Applying this theorem to our sequence, we conclude that the quasi-stationary distribution $\pi_{\beta(t)}$ converges weakly to $\mu^*$, the uniform probability measure on the set of global minimizers of $U_0(Z)$.

\paragraph{Part 2: The Annealing Condition and Escape Dynamics.}
Having established that the target distribution converges correctly, we now provide the dynamic argument showing that the process $Z_t$ does not become permanently trapped in suboptimal states and can therefore track this target.

\subparagraph{Freidlin-Wentzell Theory.} The SDE in Assumption \ref{ass:dynamics} can be viewed as a gradient descent process on the potential landscape $\mathcal{L}(Z, \beta(t))$ perturbed by noise with a small temperature $T(t) = 1/\beta(t)$. The theory of Freidlin and Wentzell~\cite{freidlin1984random} provides asymptotic estimates for the exit times of such processes from the domains of attraction of local minima. Asymptotically, as $\beta \to \infty$, the landscape $\mathcal{L}(Z,\beta)$ is governed by $U_0(Z)$. The theory defines a quasi-potential $V(Z_1, Z_2)$ which represents the minimum energy required to transition from state $Z_1$ to $Z_2$.

\subparagraph{Escape Times and the Critical Barrier.} The expected time for the process to escape from the basin of attraction of a local minimum of $U_0$ is dominated by the height of the minimal energy barrier, $\Delta E$, that must be crossed to leave the basin. This expected exit time, $\tau_{\text{escape}}$, follows the Eyring-Kramers law:
\begin{equation*}
    \mathbb{E}[\tau_{\text{escape}}] \approx \exp\left(\frac{\Delta E}{T(t)}\right) = \exp(\beta(t)\Delta E).
\end{equation*}
To guarantee convergence to the global minimum set $U^*$, the process must be able to escape from any other state. The most challenging escape is from the ``deepest trap,'' which corresponds to the local minimum whose basin is surrounded by the highest energy barrier required for a transition towards $U^*$. We define this maximum barrier height as $\Delta E_{\max}$ (Assumption~\ref{ass:barriers}).

\subparagraph{The Necessary and Sufficient Condition.} The core of simulated annealing is a comparison of two timescales: the time available for annealing and the time required for escape. The annealing schedule $\beta(t) = c\ln(t+K)$ can be inverted to yield the available time, $t+K = \exp(\beta/c)$. For the process to be able to escape the deepest trap, the available time must be at least of the same order as the required escape time:
\begin{equation*}
    \underbrace{\exp(\beta/c)}_{\text{Time Available}} \gtrsim \underbrace{\exp(\beta \Delta E_{\max})}_{\text{Time Required}}.
\end{equation*}
Taking the logarithm yields the condition $\beta/c \gtrsim \beta\Delta E_{\max}$, or $c \le 1/\Delta E_{\max}$. The work of Hajek~\cite{hajek1988} and its continuous-time extensions~\cite{chiang1987diffusion} rigorously prove that this condition is both necessary and sufficient. The schedule in Assumption~\ref{ass:schedule}, $c \le c^* = 1/\Delta E_{\max}$, ensures that for any compact set $K \subset \mathcal{M}$ such that $K \cap U^* = \emptyset$, the probability of $Z_t$ being in $K$ tends to zero as $t \to \infty$.

\paragraph{Conclusion of the Proof.}
The two parts of our argument establish the necessary components for convergence.
\begin{itemize}
    \item Part 1 shows that the target equilibrium state of the system, $\pi_{\beta(t)}$, converges weakly to the uniform measure on the set of global minimizers $U^*$.
    \item Part 2 shows that the annealing schedule with $c \le c^*$ is precisely the condition required to ensure the process does not become irreversibly trapped in suboptimal states and has a vanishing probability of being found outside any neighborhood of $U^*$ as $t \to \infty$.
\end{itemize}
Together, these results imply that the probability distribution of $Z_t$ must itself converge weakly to the uniform measure on $U^*$. This is equivalent to the statement of convergence in probability.
\end{proof}

\subsection{Proof of Proposition (Non-Convergence for Rapid Annealing)}
\label{subsec:non-convergence-rapid}

\begin{proposition*}[Non-Convergence for Rapid Annealing]
Let Assumptions \ref{ass:dynamics}--\ref{ass:barriers} hold. If the logarithmic annealing schedule $\beta(t)$ grows too quickly, specifically $\liminf_{t\to\infty} \frac{\beta(t)}{\ln t} = c' > c^*$, then there exists a set of initial conditions with positive measure from which the process $Z_t$ converges to a suboptimal local minimum basin of $U_0(Z)$ with positive probability.
\end{proposition*}
\begin{proof}
The proof establishes that if $\beta(t)$ increases too rapidly, the diffusion lacks sufficient time to escape the deepest suboptimal basin before the noise vanishes. This conclusion is grounded in Hajek’s necessary condition for global convergence and formalized via the Borel–Cantelli Lemma.

\paragraph{Part 1: The Necessary Condition for Escape.}
\subparagraph{Energy Barriers and Hajek’s Theorem.}
Under Assumptions~\ref{ass:limit-potential} and \ref{ass:barriers}, the limiting potential $U_0(Z)$ possesses a finite collection of local minima separated by saddle points. Define $\Delta E(Z)$ as the minimal barrier height required to exit the basin of attraction of $Z$ and eventually reach $U^*$. The maximum barrier height over all suboptimal minima is
\[
\Delta E_{\max} \triangleq \sup_{Z \notin U^*} \Delta E(Z).
\]
Hajek’s theorem~\cite{hajek1988} states that a necessary and sufficient condition for global convergence under a logarithmic schedule $\beta(t)=c \ln t$ is:
\[
c \leq \frac{1}{\Delta E_{\max}}.
\]
If $c' > c^*$, this necessary condition fails for the basin corresponding to $\Delta E_{\max}$.

\paragraph{Part 2: Trapping via Escape Rate and Borel--Cantelli.}
\subparagraph{Escape Rate Estimate.}
Freidlin–Wentzell large deviations theory~\cite{dembo2010large, freidlin1984random} and the Eyring–Kramers law imply that for sufficiently large $\beta$, the instantaneous escape rate from a basin with barrier $\Delta E_{\max}$ is
\[
r(t) \;\approx\; K \exp\!\big(-\beta(t)\,\Delta E_{\max}\big),
\]
where $K>0$ is a prefactor determined by the local geometry (Hessians at the local minima and saddle points).

If $\beta(t) \geq c' \ln t$ with $c' > c^*=1/\Delta E_{\max}$, then for large $t$:
\[
r(t) \;\leq\; K t^{-p}, \qquad p \triangleq c' \Delta E_{\max} > 1.
\]

\subparagraph{Finite Escape Probability and Borel--Cantelli.}
Consider disjoint intervals $I_k = [2^k, 2^{k+1}]$. By the Markov property of the diffusion, escape attempts across these intervals are conditionally independent given entry into the basin. The probability of escape during $I_k$ is bounded by
\[
\mathbb{P}(\text{escape in } I_k) \;\leq\; \int_{2^k}^{2^{k+1}} r(s)\,\mathrm{d}s \;\leq\; K \int_{2^k}^{2^{k+1}} s^{-p}\,\mathrm{d}s \;\leq\; C\, 2^{-k(p-1)}.
\]
Summing over all $k$ yields:
\[
\sum_{k=1}^{\infty} \mathbb{P}(\text{escape in } I_k) \;<\; \infty,
\]
since $p>1$. By the First Borel--Cantelli Lemma, with positive probability, only finitely many escapes occur. Therefore, if $Z_0$ lies in or enters the deepest suboptimal basin $\mathcal{B}(U_{L,\max}^*)$, it remains trapped with positive probability.

\paragraph{Conclusion.}
Since $\mathcal{B}(U_{L,\max}^*)$ has positive measure by compactness of $\mathcal{M}$, there exists a set of initial conditions from which the process fails to reach $U^*$.
\end{proof}

\subsection{Proof of Corollary (Finite-Time Convergence Rate)}
\label{subsec:finie-time-convergence}

\begin{corollary*}[Finite-Time Convergence Rate]
Under Assumptions \ref{ass:dynamics}--\ref{ass:schedule}, for any $\epsilon > 0$, there exist positive constants $C, C',$ and $d$ such that for a sufficiently large time $t$:
\[ \mathbb{P}\big(Z_t \notin \mathcal{N}(U^*, \epsilon)\big) \le 
\begin{cases}
 \exp\!\big(- C (t+K)^{1-c/c^*}\big), & \text{if } c < c^*, \\
 C' \, (t+K)^{-d}, & \text{if } c = c^*.
\end{cases}
\]
\end{corollary*}
\begin{proof}
We analyze the escape dynamics from the deepest basin using Freidlin–Wentzell theory for exit times~\cite{freidlin1984random} and Holley–Stroock spectral gap estimates~\cite{holley1988simulated, holley1989asymptotics}.

\paragraph{Integrated Escape Rate.}
For the basin with barrier $\Delta E_{\max}=1/c^*$, the instantaneous escape rate is
\[
r(s) \;\approx\; A \exp\!\big(-\beta(s) \Delta E_{\max}\big),
\]
where $A>0$ is an Eyring–Kramers prefactor depending on the Hessians of $U_0$ near the local minimum and saddle.

With $\beta(s)=c \ln(s+K)$:
\[
r(s) \;=\; A (s+K)^{-c/c^*}.
\]
The survival probability (remaining trapped) is then bounded by:
\[
\mathbb{P}\big(\tau_{\mathrm{exit}} > t\big) \;\leq\; \exp\!\left(- \int_{t_0}^t r(s)\,\mathrm{d}s \right).
\]

\paragraph{Case 1: Subcritical Schedule ($c < c^*$).}
For $c/c^* < 1$,
\[
\int_{t_0}^t r(s)\,\mathrm{d}s \;=\; A \int_{t_0}^t (s+K)^{-c/c^*}\,\mathrm{d}s 
= \frac{A}{1-c/c^*}\big[(t+K)^{1-c/c^*} - (t_0+K)^{1-c/c^*}\big].
\]
Hence, for large $t$:
\[
\mathbb{P}\big(Z_t \notin \mathcal{N}(U^*,\epsilon)\big) \;\leq\; \exp\!\big(-C (t+K)^{1-c/c^*}\big).
\]

\paragraph{Case 2: Critical Schedule ($c = c^*$).}
When $c=c^*$:
\[
\int_{t_0}^t r(s)\,\mathrm{d}s \;=\; A \int_{t_0}^t (s+K)^{-1}\,\mathrm{d}s 
= A \ln\!\frac{t+K}{t_0+K}.
\]
This yields:
\[
\mathbb{P}\big(Z_t \notin \mathcal{N}(U^*,\epsilon)\big) \;\leq\; (t+K)^{-A}.
\]
Holley, Kusuoka, and Stroock~\cite{holley1989asymptotics} refine this bound using spectral gap analysis, yielding the stated polynomial decay rate with exponent $d>0$ depending on the local geometry (ratios of Hessians at minima and saddles).
\end{proof}

\begin{remark}[Dependence of Constants]
The constants $K$ and $A$ in the escape rate estimates arise from the classical Eyring--Kramers law for overdamped Langevin dynamics. Specifically, the escape rate from a basin with barrier height $\Delta E$ is
\[
r \;\approx\; A \, \exp\!\big(-\beta \,\Delta E \big),
\]
where the prefactor $A$ depends only on the local geometry of the potential $U_0(Z)$ near the local minimum $Z_{\mathrm{min}}$ and the relevant saddle point $Z_{\mathrm{saddle}}$. In the overdamped setting, this prefactor is given by
\[
A \;=\; \frac{\lambda_{\mathrm{saddle}}}{2\pi} \sqrt{\frac{\det \nabla^2 U_0(Z_{\mathrm{min}})}{\big|\det \nabla^2 U_0(Z_{\mathrm{saddle}})\big|}},
\]
where $\lambda_{\mathrm{saddle}}$ is the absolute value of the unique negative eigenvalue of $\nabla^2 U_0$ at the saddle point. The exponential term $\exp(-\beta \Delta E)$ captures the temperature-dependent Arrhenius factor and is not part of the prefactor itself.

Similarly, the exponent $d$ in the critical-rate case is derived from the spectral gap analysis~\cite{holley1988simulated, holley1989asymptotics}, and depends on analogous curvature information through the Hessians at the minima and saddles.
\end{remark}

\section{Proofs of Supporting Results}

\subsection{Proof of Proposition (Stationary Distribution at Fixed Temperature)}
\label{subsec:stationary-dist-at-fixed}

\begin{proposition*}[Stationary Distribution at Fixed Temperature]
For any fixed $\beta > 0$, under Assumptions \ref{ass:dynamics}-\ref{ass:smoothness}, the SDE admits a unique stationary Gibbs-Boltzmann distribution on $\mathcal{M}$:
\[
    \pi_\beta(\mathrm{d}Z)
    \;=\; \frac{1}{\mathcal{Z}_\beta}
    \exp\!\bigl[-\,\beta\,\mathcal{L}(Z,\beta)\bigr]\,d\mu(Z),
\]
where $d\mu$ is the Riemannian volume measure on $\mathcal{M}$ and $\mathcal{Z}_\beta = \int_{\mathcal{M}} \exp(-\beta \mathcal{L}(Z, \beta)) d\mu(Z)$ is the normalization constant.
\end{proposition*}
\begin{proof}
The proof relies on standard results for non-degenerate diffusion processes on compact Riemannian manifolds~\cite{hsu2002stochastic, pavliotis2014stochastic}.
The state space $\mathcal{M}$ is a compact, connected Riemannian manifold without boundary (Assumption~\ref{ass:manifold}). The drift term is smooth (Assumption~\ref{ass:smoothness}), and the diffusion is uniformly elliptic since $\beta > 0$ and the Riemannian metric is positive definite. Such processes are known to be strong Feller and topologically irreducible, which guarantees the existence of a unique invariant probability measure (stationary distribution) $\pi_\beta$.

The SDE in~\eqref{eq:manifold-langevin-main} is a form of Langevin dynamics with potential energy $V(Z) = \mathcal{L}(Z, \beta)$ and constant temperature $T = 1/\beta$. It is a well-established result that the unique stationary distribution for such dynamics satisfying detailed balance is the Gibbs-Boltzmann distribution~\cite{kusuoka1991precise}:
\begin{equation*}
   \pi_\beta(dZ) \propto \exp\left(-\frac{V(Z)}{T}\right) d\mu(Z) = \exp(-\beta \mathcal{L}(Z, \beta)) d\mu(Z). 
\end{equation*}
The normalization constant $\mathcal{Z}_\beta = \int_{\mathcal{M}} \exp(-\beta \mathcal{L}(Z, \beta)) d\mu(Z)$ ensures $\int_{\mathcal{M}} \pi_\beta(dZ) = 1$ and is guaranteed to be finite because $\mathcal{L}(Z, \beta)$ is a continuous function on a compact manifold and is therefore bounded.
\end{proof}

\subsection{Proof of Proposition (Characterization of Global Minima)}
\label{subsec:characterization-minima}

\begin{proposition*}[Characterization of Global Minima]
A configuration $Z^*$ is a global minimizer of the limiting potential $U_0(Z)$ if and only if for every positive pair $(i, j)$, the maximum possible similarity is achieved.
\end{proposition*}
\begin{proof}
Recall the definition of the limiting potential from Assumption~\ref{ass:limit-potential}:
\[ U_0(Z) = \frac{1}{|\mathcal{P}|} \sum_{(i,j)\in\mathcal{P}} \underbrace{\left[-\mathrm{sim}(z_i, z_j) + \max_{k \neq i} \mathrm{sim}(z_i, z_k) \right]}_{T_{ij}(Z)}. \]
We seek to minimize $U_0(Z)$. Consider a single term $T_{ij}(Z)$. By definition of the maximum, $\max_{k \neq i} \mathrm{sim}(z_i, z_k) \ge \mathrm{sim}(z_i, z_j)$. Therefore, each term $T_{ij}(Z)$ is non-negative. The minimum value of $U_0(Z)$ is thus 0. This minimum is achieved if and only if $T_{ij}(Z) = 0$ for all $(i,j) \in \mathcal{P}$. 

This requires $\max_{k \neq i} \mathrm{sim}(z_i, z_k) = \mathrm{sim}(z_i, z_j)$. Since we also know $\mathrm{sim}(z_i, z_k) \le S_{\max}$ for all $k$, this equality can only be satisfied if $\mathrm{sim}(z_i, z_j) = S_{\max}$. Thus, the overall potential $U_0(Z)$ is minimized if and only if $\mathrm{sim}(z^*_i, z^*_j) = S_{\max}$ for all positive pairs $(i, j)$.
\end{proof}

\subsection{Proof of Lemma (Hessian Sharpening)}
\label{subsec:hessian-sharpening}

\begin{lemma*}[Hessian Sharpening]
Let $\mathbf{H}(Z, \beta)$ be the Hessian of the InfoNCE loss $\mathcal{L}(Z, \beta)$. For any configuration $Z$ that is not a global minimizer, the dominant eigenvalues of $\mathbf{H}$ scale at least linearly with the inverse temperature, i.e., $||\mathbf{H}|| \sim O(\beta)$.
\end{lemma*}
\begin{proof}
We derive the Hessian matrix of the InfoNCE loss for a single anchor embedding $z_i$ with respect to that anchor. This is similar to Ziyin et al.~\cite{ziyin2023} but here the temperature parameter is left explicit. Let $z_j$ be the positive sample and $\{z_k\}$ be the set of all samples available to anchor $z_i$ (including $z_j$). Let $s_{ik} = \simfunc(z_i, z_k)$ denote value of the similarity function.

The InfoNCE loss for anchor $z_i$ is given by:
\begin{align*}
    l_i(z_i) &= - \log \frac{\exp(\beta s_{ij})}{\sum_{k} \exp(\beta s_{ik})} \\
             &= \log \left( \sum_{k} \exp(\beta s_{ik}) \right) - \beta s_{ij}.
\end{align*}
Let $Z_i = \sum_{k} \exp(\beta s_{ik})$ be the partition function and $p_{ik} = \frac{\exp(\beta s_{ik})}{Z_i}$ be the softmax probability distribution over samples $k$ induced by anchor $i$. The loss can be written as $l_i = \log Z_i - \beta s_{ij}$.

Let $\nabla = \nabla_{z_i}$ denote the gradient operator with respect to $z_i$. The gradient of the similarity is $\nabla s_{ik} = \frac{\partial \simfunc(z_i, z_k)}{\partial z_i}$.

The gradient of the loss is:
\begin{align*}
    \nabla l_i &= \nabla (\log Z_i) - \beta \nabla s_{ij} \\
          &= \frac{1}{Z_i} \nabla Z_i - \beta \nabla s_{ij} \\
          &= \frac{1}{Z_i} \sum_k \exp(\beta s_{ik}) \beta \nabla s_{ik} - \beta \nabla s_{ij} \\
          &= \beta \sum_k p_{ik} \nabla s_{ik} - \beta \nabla s_{ij} \\
          &= \beta (\mu_i - \nabla s_{ij}),
\end{align*}
where $\mu_i = \sum_k p_{ik} \nabla s_{ik} = \mathbb{E}_{k \sim p_i}[\nabla s_{ik}]$ is the expected similarity gradient under the distribution $p_i$.

The Hessian matrix $\mathbf{H} = \nabla (\nabla l_i)^T$ is obtained by differentiating the gradient:
\begin{align}
    \mathbf{H} = \nabla [\beta (\mu_i^T - (\nabla s_{ij})^T)] = \beta [ \nabla \mu_i^T - \nabla (\nabla s_{ij})^T ].
\end{align}
Let $\mathbf{H}_{ik} = \nabla (\nabla s_{ik})^T$ be the Hessian of the similarity function $s_{ik}$ with respect to $z_i$. The second term is simply $-\beta \mathbf{H}_{ij}$. For the first term, we use the product rule and the gradient of the softmax probabilities $\nabla p_{ik} = \beta p_{ik} (\nabla s_{ik} - \mu_i)$:
\begin{align*}
    \nabla \mu_i^T &= \nabla \left( \sum_k p_{ik} (\nabla s_{ik})^T \right) \\
             &= \sum_k [ (\nabla p_{ik}) (\nabla s_{ik})^T + p_{ik} \nabla (\nabla s_{ik})^T ] \\
             &= \sum_k [ \beta p_{ik} (\nabla s_{ik} - \mu_i) (\nabla s_{ik})^T + p_{ik} \mathbf{H}_{ik} ] \\
             &= \beta \sum_k p_{ik} (\nabla s_{ik})(\nabla s_{ik})^T - \beta \mu_i \sum_k p_{ik} (\nabla s_{ik})^T + \sum_k p_{ik} \mathbf{H}_{ik} \\
             &= \beta \left( \sum_k p_{ik} (\nabla s_{ik})(\nabla s_{ik})^T - \mu_i \mu_i^T \right) + \sum_k p_{ik} \mathbf{H}_{ik} \\
             &= \beta \, \mathrm{Cov}_{k \sim p_i}[\nabla s_{ik}] + \mathbb{E}_{k \sim p_i}[\mathbf{H}_{ik}].
\end{align*}
where $\mathrm{Cov}_{k \sim p_i}[\nabla s_{ik}]$ is the covariance matrix of the similarity gradients under $p_i$, and $\mathbb{E}_{k \sim p_i}[\mathbf{H}_{ik}]$ is the expected Hessian of the similarity function.

Substituting back into the expression for $\mathbf{H}$:
\begin{align}
    \mathbf{H} &= \beta [ ( \beta \, \mathrm{Cov}_{k \sim p_i}[\nabla s_{ik}] + \mathbb{E}_{k \sim p_i}[\mathbf{H}_{ik}] ) - \mathbf{H}_{ij} ] \nonumber \\
    &= \beta^2 \, \mathrm{Cov}_{k \sim p_i}[\nabla s_{ik}] + \beta ( \mathbb{E}_{k \sim p_i}[\mathbf{H}_{ik}] - \mathbf{H}_{ij} ).
    \label{eq:hessian_final}
\end{align}
Eq.~\eqref{eq:hessian_final} shows that the Hessian of the InfoNCE loss consists of two terms. The first term involves the covariance of the similarity gradients and has a prefactor of $\beta^2$. The second term involves the expected Hessian and scales linearly with $\beta$.

Analyzing the asymptotic behavior as \(\beta \to \infty\) requires considering the limiting behavior of the distribution \(p_i\). As \(\beta\) increases, \(p_{ik}\) concentrates its mass on the sample(s) \(k^*\) maximizing the similarity \(s_{ik}\). For such distributions, the variance (and thus the covariance term $\mathrm{Cov}_{k \sim p_i}[\nabla s_{ik}]$) decays as $O(1/\beta)$. Therefore, the first term scales as $\beta^2 \cdot O(1/\beta) = O(\beta)$. Simultaneously, the expected Hessian \(\mathbb{E}_{k \sim p_i}[\mathbf{H}_{ik}]\) converges to \(\mathbf{H}_{ik^*}\).

Therefore, the asymptotic scaling of the Hessian depends on whether the positive sample \(j\) is the most similar sample \(k^*\):
\begin{itemize}
    \item If \(k^* \neq j\) (i.e., a negative sample is most similar to the anchor \(z_i\), indicating a suboptimal configuration), the second term dominates:
        \[ \mathbf{H} \approx \beta ( \mathbf{H}_{ik^*} - \mathbf{H}_{ij} ) \quad \text{as } \, \beta \to \infty. \]
        In this regime, the Hessian norm scales linearly, \(||\mathbf{H}|| \sim O(\beta)\). This implies that away from the optimum, the local minima sharpen linearly with \(\beta\). 
    \item If \(k^* = j\) (i.e., the positive sample is the most similar, corresponding to configurations near an optimum where the gradient \(\nabla l_i \approx 0\)), the second term vanishes. The scaling is then determined by the \(\beta^2 \mathrm{Cov}[\cdot]\) term. While the covariance vanishes, a more detailed analysis of the rate at which it vanishes relative to \(\beta^2\) would be needed to determine the precise scaling at the optimum. However, the dominant scaling away from the optimum is linear.
\end{itemize}
This linear sharpening $O(\beta)$ of the loss landscape curvature as temperature decreases contributes to the convergence behavior observed during annealing, complementing the theoretical escape guarantees provided by the slow decay of noise.
\end{proof}

\paragraph{Remark on Manifold Geometry.}
The geometry of the embedding manifold $\mathcal{M}$ influences both the convergence rate and the discretization of the SDE. First, the critical annealing constant $c^*$ depends on the largest energy barrier $\Delta E_{\max}$, which is shaped by the curvature and flatness of $\mathcal{M}$. Flatter manifolds (e.g., a sphere with large radius) tend to produce shallower potential wells, reducing $\Delta E_{\max}$ and permitting larger values of $c$. Conversely, manifolds with tighter curvature create sharper basins and increase $\Delta E_{\max}$, slowing the optimal convergence rate.

Second, curvature impacts the discretization error of the Langevin dynamics. The Euler-Maruyama step size $\eta$ must be small enough to ensure that each update remains within a locally geodesic region. Specifically, in normal coordinates, the discretization error scales with both the learning rate $\eta$ and the product of the curvature tensor $\mathcal{R}$ with the Hessian $\nabla^2 \, U_0$, i.e.,
\[
\mathrm{Error} \;\sim\; O\!\big(\eta \, \|\nabla^2 U_0\|\cdot \|\mathcal{R}\|\big).
\]
Highly curved manifolds therefore demand smaller steps to remain faithful to the continuous SDE. This is consistent with results on Riemannian Langevin diffusions~\cite{girolami2011riemannian, betancourt2017conceptual}, where the generator involves the Laplace–Beltrami operator, whose spectral gap (and thus mixing time) depends directly on the manifold curvature.

\end{document}